\documentclass[11pt,a4paper]{article}
\usepackage[hyperref]{acl2018}

\usepackage{times}
\usepackage{latexsym}

\usepackage{url}

\usepackage{lineno}
\modulolinenumbers[5]

\usepackage{graphicx}

\usepackage{caption}
\usepackage[subrefformat=parens, margin=5pt]{subcaption}
\usepackage{multirow}

\usepackage{amsmath}
\usepackage{amssymb}
\usepackage{amsthm}
\usepackage{color}
\usepackage[T1]{fontenc}

\newcommand{\average}[1]{\ensuremath{\langle#1\rangle} }

\newcommand{\numall}{1142}

\newcommand{\pr}{0.48}

\newcommand{\argmin}{\mathop{\rm arg~min}\limits}

\newtheorem{prop}{Proposition}

\newcommand{\secref}[1]{\S \ref{#1}}
\newcommand{\figref}[1]{Figure~\ref{#1}}
\newcommand{\tabref}[1]{Table~\ref{#1}}

\aclfinalcopy 
\def\aclpaperid{141} 

\title{Taylor's Law for Human Linguistic Sequences }

\author{Tatsuru Kobayashi\thanks{~ {\tt kobayashi@cl.rcast.u-tokyo.ac.jp}} \\
	Graduate School of \\ Information Science and Technology,\\ University of Tokyo\\
	7-3-1 Hongo, Bunkyo-ku\\ Tokyo 113-8656\\ Japan \\\And 
	Kumiko Tanaka-Ishii\thanks{~ {\tt kumiko@cl.rcast.u-tokyo.ac.jp}} \\
	Research Center for \\ Advanced Science and Technology,\\ University of Tokyo\\
	4-6-1 Komaba, Meguro-ku\\ Tokyo 153-8904\\ Japan \\
}

\date{}

\begin{document}

\maketitle

\begin{abstract}
Taylor's law describes the fluctuation characteristics underlying a system in which the 
variance of an event within a time span grows by a power law with respect to the mean. 
Although Taylor's law has been applied in many natural and social systems, its 
application for language has been scarce. This article describes a new
quantification of Taylor's law in natural language and reports an  analysis of over 
1100 texts across 14 languages. The Taylor exponents of written natural language texts 
were found to exhibit almost the same value. The exponent was also compared for other 
language-related data, such as the child-directed speech, music, and programming 
language code. The results show how the Taylor exponent serves to quantify the 
fundamental structural complexity underlying linguistic time series. The article also 
shows the applicability of these findings in evaluating language models.
\end{abstract}

\section{Introduction}
\label{sec:introduction}
Taylor's law characterizes how the variance of the number of events for a given time 
and space grows with respect to the mean, forming a power law. It is a 
quantification method for the clustering behavior of a system. Since the pioneering 
studies of this concept \citep{taylor-smith,taylor-nature}, a substantial number of 
studies have been conducted across various domains, including ecology, life science, physics, 
finance, and human dynamics, as well summarized in \citep{taylor}. More recently, 
\citet{Cohen7749} reported Taylor exponents for random sampling from various 
distributions, and \citet{Calif2015} reported Taylor's law in wind energy data using a 
non-parametric regression. Those two papers also refer to research about Taylor's law in 
a wide range of fields.

Despite such diverse application across domains, there has been little analysis based on 
Taylor's law in studying natural language. The only such report, to the best of our 
knowledge, is \citet{altmann14-taylor}, but they measured the mean and variance by 
means of the vocabulary size within a document. This approach essentially differs from 
the original concept of Taylor analysis, which fundamentally counts the number of 
events, and thus the theoretical background of Taylor's law as presented in 
\citet{taylor} cannot be applied to interpret the results.

For the work described in this article, we applied Taylor's law for texts, in a manner 
close to the original concept. We considered lexical fluctuation within texts, which 
involves the co-occurrence and burstiness of word alignment. The results can thus be 
interpreted according to the analytical results of Taylor's law, as described later. We 
found that the Taylor exponent is indeed a characteristic of texts and is universal across 
various kinds of texts and languages. These results are shown here for data including 
over 1100 single-author texts across 14 languages and large-scale newspaper data.

Moreover, we found that the Taylor exponents for other symbolic sequential data, 
including child-directed speech, programming language code, and music, differ from 
those for written natural language texts, thus distinguishing different kinds of data 
sources. The Taylor exponent in this sense could categorize and quantify the structural 
complexity of language. The Chomsky hierarchy \citep{chomsky} is, of course, the 
most important framework for such categorization. The Taylor exponent is another way 
to quantify the complexity of natural language: it allows for continuous quantification 
based on lexical fluctuation.

Since the Taylor exponent can quantify and characterize one aspect of natural language, 
our findings are applicable in computational linguistics to assess language models. At 
the end of this article, in \secref{sec:lm}, we report how the most basic character-based 
long short-term memory (LSTM) unit produces texts with a Taylor exponent of 0.50, 
equal to that of a sequence of independent and identically distributed random 
variables (an i.i.d. sequence). This shows how such models are limited in producing 
consistent co-occurrence among words, as compared with a real text. Taylor analysis 
thus provides a possible direction to reconsider the limitations of language models.

\section{Related Work}
This work can be situated as a study to quantify the complexity underlying texts. As 
summarized in \citep{Tanaka-Ishii2015}, measures for this purpose include the entropy 
rate \citep{entropy16, entropy17} and those related to the scaling behaviors of 
natural language. Regarding the latter, certain power laws are known to hold universally 
in linguistic data. The most famous among these are Zipf's law \citep{zipf} and Heaps' 
law \citep{heaps}. Other, different kinds of power laws from Zipf's law are obtained 
through various methods of fluctuation analysis, but the question of how to quantify the 
fluctuation existing in language data has been controversial. Our work is situated as one such 
case of fluctuation analysis.

In real data, the occurrence timing of a particular event is often biased in a bursty, 
clustered manner, and fluctuation analysis quantifies the degree of this bias. Originally, 
this was motivated by a study of how floods of the Nile River occur in clusters (i.e., 
many floods coming after an initial flood) \citep{hurst}. Such clustering phenomena 
have been widely reported in both natural and social domains \citep{taylor}.

Fluctuation analysis for language originates in \citep{Ebeling1994}, which applied the 
approach to characters. That work corresponds to observing the average of the variances 
of each character's number of occurrences within a time span. Their method is strongly 
related to ours but different from two viewpoints: (1) Taylor analysis considers the 
variance with respect to the mean, rather than time; and (2) Taylor analysis does not 
average results over all elements. Because of these differences, the method in 
\citep{Ebeling1994} cannot distinguish real texts from an i.i.d. process when applied to 
word sequences \citep{takahashi}.


Event clustering phenomena cause a sequence to resemble itself in a self-similar manner. 
Therefore, studies of the fluctuation underlying a sequence can take another form of 
{\em long-range correlation analysis}, to consider the similarity between two 
subsequences underlying a time series. This approach requires a function to calculate 
the similarity of two sequences, and the autocorrelation function (ACF) is the main 
function considered. Since the ACF only applies to numerical data, both 
\citet{Altmann2009} and \citet{plosone16} applied long-range correlation analysis by 
transforming text into intervals and showed how natural language texts are long-range 
correlated. Another recent work \citep{Lin_2016} proposed using mutual information 
instead of the ACF. Mutual information, however, cannot detect the long-range 
correlation underlying texts. All these works studied correlation phenomena via only a 
few texts and did not show any underlying universality with respect to data and 
language types. One reason is that analysis methods for long-range correlation are 
non-trivial to apply to texts.

Overall, the analysis based on Taylor's law in the present work belongs to the former 
approach of fluctuation analysis and shows the law's vast applicability and stability for 
written texts and even beyond, quantifying universal complexity underlying human 
linguistic sequences.

\section{Measuring the Taylor Exponent}
\label{sec:theory}
\subsection{Proposed method}
\label{subsec:method}
Given a set of elements $W$ (words), let $X = X_1,X_2,\ldots,X_{N}$ be a discrete 
time series of length $N$, where $X_i \in W$ for all $i=1, 2, \ldots, N$, i.e., each 
$X_{i}$ represents a word. For a given segment length $\Delta t \in \mathbb{N}$ (a 
positive integer), a data sample $X$ is segmented by the length $\Delta t$. The number 
of occurrences of a specific word $w_k \in W$ is counted for every segment, and the 
mean $\mu_k$ and standard deviation $\sigma_k$ across segments are obtained. 
Doing this for all word kinds $w_{1}, \ldots, w_{|W|} \in W$ gives the distribution
of $\sigma$ with respect to $\mu$. Following a previous work \citep{taylor}, in this article Taylor's law 
is defined to hold when $\mu$ and $\sigma$ are correlated by a power law in the 
following way:
\begin{equation}
 \sigma \propto \mu^\alpha .
\end{equation}
Experimentally, the Taylor exponent $\alpha$ is known to take a value within the range 
of $0.5 \leq \alpha \leq 1.0$ across a wide variety of domains as reported in 
\citep{taylor}, including finance, meteorology, agriculture, and biology. Mathematically, 
it is analytically proven that $\alpha = 0.5$ for an i.i.d process, and the proof is 
included as Supplementary Material.

On the other hand, $\alpha=1.0$ when all segments {\em always contain the same 
proportion of the elements of $W$}. For example, suppose that $W = \{a, b\}$. If 
$b$ always occurs twice as often as $a$ in all segments (e.g., three $a$ and six $b$ in 
one segment, two $a$ and four $b$ in another, etc.), then both the mean and standard 
deviation for $b$ are twice those for $a$, so the exponent is 1.0.

In a real text, this cannot occur for all $W$, so $\alpha < 1.0$ for natural language text. 
Nevertheless, for a subset of words in $W$, this could happen, especially for a 
template-like sequence. For instance, consider a programming statement: {\tt while (i < 
1000) do i--}. Here, the words {\tt while} and {\tt do} always occur once, whereas {\tt 
i} always occurs twice. This example shows that the exponent indicates how 
consistently words depend on each other in $W$, i.e., how words co-occur 
systematically in a coherent manner, thus indicating that the Taylor exponent is partly 
related to grammaticality.

To measure the Taylor exponent $\alpha$, the mean and standard deviation are 
computed for every word kind\footnote{ In this work, words are not lemmatized, e.g. 
``say,'' ``said,'' and ``says'' are all considered different words. This was chosen so in this 
work because the Taylor exponent considers systematic co-occurrence of words, and 
idiomatic phrases should thus be considered in their original forms. } and then plotted in 
log-log coordinates. The number of points in this work was the number of different 
words. We fitted the points to a linear function in log-log coordinates by the 
least-squares method. We naturally took the logarithm of both $c \mu^\alpha$ and 
$\sigma$ to estimate the exponent, because Taylor's law is a power law. The coefficient 
$\hat{c,}$ and exponent $\hat{\alpha}$ are then estimated as the following:
\begin{eqnarray}
\scriptsize
 \hat{c}, \hat{\alpha} &=& \argmin_{c, \alpha} \epsilon(c, \alpha), \nonumber \\
\epsilon(c, \alpha) &=& \sqrt{\frac{1}{|W|}\sum_{k=1}^{|W|} (\log \sigma_k - \log c 
\mu_k^\alpha)^2}.  \nonumber
\end{eqnarray}
This fit function could be a problem depending on the distribution of errors between the 
data points and the regression line. As seen later, the error distribution seems to 
differ with the kind of data: for a random source the error seems Gaussian, and so the 
above formula is relevant, whereas for real data, the distribution is biased. Changing the 
fit function according to the data source, however, would cause other essential problems 
for fair comparison. Here, because \citet{Cohen7749} reported that most empirical 
works on Taylor's law used least-squares regression (including their own), this work 
also uses the above scheme\footnote{The code for estimating the exponent is available 
from \url{https://github.com/Group-TanakaIshii/word_taylor}.}, with the error defined as 
$\epsilon(\hat{c}, \hat{\alpha})$.


\begin{table*}[t]
  \caption{\label{tab:data}
    Data we used in this article. 
    For each dataset, 
    length is the number of words, 
    vocabulary is the number of different words. 
    For detail of the data kind, see \secref{subsec:data}.}


  \fontsize{7.5pt}{12pt}\selectfont

	\hspace{-0.5cm}
        \begin{tabular}{|l|c|c|r|r|r|r|r|r|r|} \hline
          Texts &   Language    &$\hat{\alpha}$    &    Number &  \multicolumn{3}{|c|}{Length}  & \multicolumn{3}{|c|}{Vocabulary} \\  \cline{5-10}
  & & mean &   of samples                & Mean  & Min & Max  & Mean & Min & Max \\ \hline \hline
               & English & 0.58 & 910 & 313127.4 & 185939 & 2488933 & 17237.7 & 7321 & 69812 \\ \cline{2-10}
	  & French & 0.57 & 66 & 197519.3 & 169415 & 1528177 & 22098.3 & 14106 & 57193 \\ \cline{2-10}
		& Finnish & 0.55 & 33 & 197519.3 & 149488 & 396920 & 33597.1 & 26275 & 47263 \\ \cline{2-10}
		& Chinese & 0.61 & 32 & 629916.8 & 315099 & 4145117 & 15352.9 & 9153 & 60950 \\ \cline{2-10}
		& Dutch & 0.57 & 27 & 256859.2 & 198924 & 435683 & 19159.1 & 13880 & 31595 \\ \cline{2-10}
		& German & 0.59 & 20 & 236175.0 & 184321 & 331322 & 24242.3 & 11079 & 37228 \\ \cline{2-10}
Gutenberg       & Italian & 0.57 & 14 & 266809.0 & 196961 & 369326 & 29103.5 & 18641 & 45032 \\ \cline{2-10}
		& Spanish & 0.58 & 12 & 363837.2 & 219787 & 903051 & 26111.1 & 18111 & 36507 \\ \cline{2-10}
		& Greek & 0.58 & 10 & 159969.2 & 119196 & 243953 & 22805.7 & 15877 & 31386 \\  \cline{2-10}
		& Latin & 0.57 & 2 & 505743.5 & 205228 & 806259 & 59667.5 & 28739 & 90596 \\ \cline{2-10}
		& Portuguese & 0.56 & 1 & 261382.0 & 261382 & 261382 & 24719.0 & 24719 & 24719 \\ \cline{2-10}
		& Hungarian & 0.57 & 1 & 198303.0 & 198303 & 198303 & 38384.0 & 38384 & 38384 \\ \cline{2-10}
		& Tagalog & 0.59 & 1 & 208455.0 & 208455 & 208455 & 26335.0 & 26335 & 26335 \\ \hline
	Aozora & Japanese & 0.59 & 13 & 616677.2 & 105343 & 2951320 & 19760.0 & 6620 & 49100 \\ \hline
        \hline
        Moby Dick  & English & 0.58 & 1 & 254655.0 & 254655 & 254655 & 20473.0 & 20473 & 20473 \\ \hline
	Hong Lou Meng  & Chinese & 0.59 & 1 & 701256.0 & 701256 & 701256 & 18451.0 & 18451 & 18451 \\ \hline
       Les Miserables  & French & 0.57 & 1 &  691407.0 & 690417 & 690417 & 31956.0 &   31956 &   31956 \\ \hline   
	Shakespeare (All) & English & 0.59 & 1 & 1000238.0 & 1000238 & 1000238 & 40840.0 & 40840 & 40840 \\ \hline
       \hline
        
	WSJ & English & 0.56 & 1 & 22679513.0 & 22679513 & 22679513 & 137467.0 & 137467 & 137467 \\ \hline
		NYT & English & 0.58 & 1 & 1528137194.0 & 1528137194 & 1528137194 & 3155495.0 & 3155495& 3155495 \\ \hline
		People's Daily &  Chinese &0.58 & 1 & 19420853.0 & 19420853 & 19420853 & 172140.0 & 172140& 172140 \\ \hline
		Mainichi & Japanese & 0.56 & 24 (yrs) & 31321594.3 & 24483331 & 40270706 & 145534.5 & 127290 & 169270 \\ \hline
                \hline
		enwiki8 & tag-annotated  & 0.63 & 1 & 14647848.0 & 14647848 & 14647848 & 1430791.0 & 1430791 & 1430791 \\ \hline  
		CHILDES & various & 0.68 & 10  & 193434.0 & 48952 & 448772 & 9908.0 & 5619 & 17893 \\ \hline
		Programs & - & 0.79 & 4 & 34161018.8 & 3697199 & 68622162 & 838907.8 & 127653 & 1545127 \\ \hline
		Music & - & 0.79 & 12 & 135993.4 & 76629 & 215480 & 9187.9 & 907 & 27043 \\ \hline
		\hline
		Moby Dick (shuffled) & - & 0.50 &10 & 254655.0 & 254655 & 254655 & 20473.0 & 20473 & 20473 \\ \hline
		Moby Dick (bigram) & - & 0.50 & 10 & 300001.0 & 300001 & 300001 & 16963.8 & 16893 & 17056 \\ \hline
                3-layer stacked LSTM & \multirow{2}{*}{(English)} & \multirow{2}{*}{0.50} & \multirow{2}{*}{1} & \multirow{2}{*}{256425.0} & \multirow{2}{*}{256425} & \multirow{2}{*}{256425} & \multirow{2}{*}{50115.0} & \multirow{2}{*}{50115} & \multirow{2}{*}{50115} \\
                (character-based)   &  &  &  & & &  &  & &  \\ \hline		
                Neural  MT & (English) & 0.57 & 1 & 623235.0 & 623235 & 623235 & 27370.0 & 27370 & 27370 \\ \hline 
                
	\end{tabular}
\end{table*}


\subsection{Data}
\label{subsec:data}
\tabref{tab:data} lists all the data used for this article. The data consisted of natural 
language texts, language-related sequences, and randomized data, listed as different 
blocks in the table. The natural language texts consisted of \numall\ single-author long 
texts (first block, extracted from Project Gutenberg and Aozora Bunko across 14 
languages\footnote{All texts above a size threshold (1 megabyte) were extracted from 
the two archives, resulting in 1142 texts.}, with the second block listing individual 
samples taken from Project Gutenberg together with the complete works of 
Shakespeare), and newspapers (third block, from the Gigaword corpus, available from 
the Linguistic Data Consortium in English, Chinese, and other major languages). Other 
sequences appear in the fourth block: the enwiki8 100-MB dump dataset (consisting of 
tag-annotated text from English Wikipedia), the 10 longest child-directed speech 
utterances in CHILDES data\footnote{Child Language Data Exchange System 
\citep{childes,childes-dutch, childes-english, childes-french, 
childes-german,childes-indonesian, childes-japanese, childes-polish, 
childes-serbian,childes-spanish, childes-swedish}} (preprocessed by extracting only 
children's utterances), four program sources (in Lisp, Haskell, C++, and Python, 
crawled from large representative archives, parsed, and stripped of natural language 
comments), and 12 pieces of musical data (long symphonies and so forth, transformed 
from MIDI into text with the software SMF2MML\footnote{\tt 
http://shaw.la.coocan.jp/smf2mml/}, with annotations removed). 

As for the randomized data listed in the last block, we took the text of {\em Moby 
Dick} and generated 10 different shuffled samples and bigram-generated sequences. We 
also introduced LSTM-generated texts to consider the utility of our findings, as 
explained in \secref{sec:lm}.

\section{Taylor Exponents for Real Data}
\begin{figure}[t]
	\centering
	\subcaptionbox{Moby Dick\label{subfig:moby}}{\includegraphics[clip, width=\pr\columnwidth]{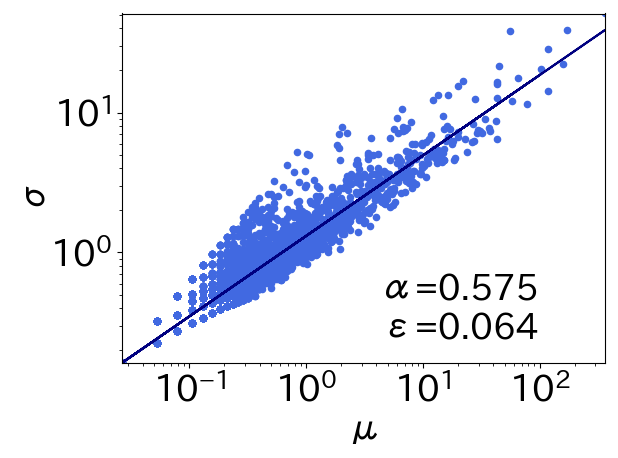}}%
	\subcaptionbox{Hong Lou Meng\label{subfig:koromu}}{\includegraphics[clip, width=\pr\columnwidth]{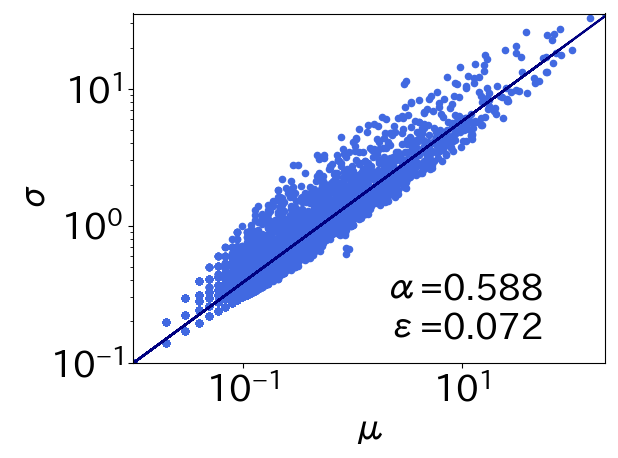}}
	\subcaptionbox{Wall Street Journal\label{subfig:wsj}}{\includegraphics[clip, width=\pr\columnwidth]{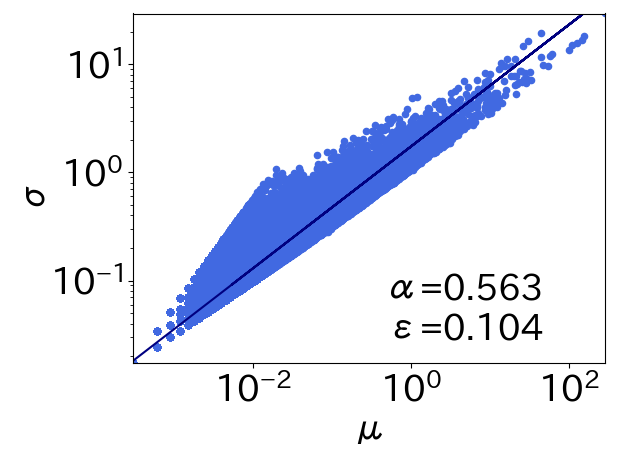}}%
	\subcaptionbox{People's Daily\label{subfig:peopledaily}}{\includegraphics[clip, width=\pr\columnwidth]{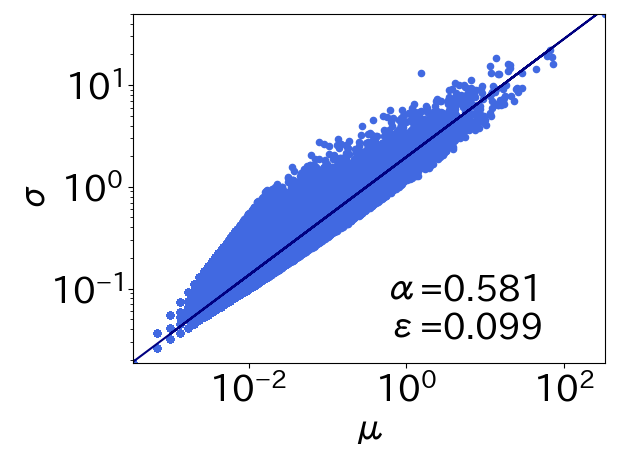}}
	\caption{\label{fig:fig1}
Examples of Taylor's law for natural language texts. {\em Moby Dick} and 
{\em Hong Lou Meng} are representative of single-author texts, and the two 
newspapers are representative of multiple-author texts, in English and Chinese, 
respectively. Each point represents a kind of word. The values of $\sigma$ and 
$\mu$ for each word kind are plotted across texts within segments of size $\Delta t = 
5620$. The Taylor exponents obtained by the least-squares method were all around 
0.58.}
\end{figure}

\figref{fig:fig1} shows typical distributions for natural language texts, with two 
single-author texts (\subref{subfig:moby} and \subref{subfig:koromu}) and two 
multiple-author texts (newspapers, \subref{subfig:wsj} and 
\subref{subfig:peopledaily}), in English and Chinese, respectively. The segment size 
was $\Delta t = 5620$ words\footnote{ In comparison, \figref{fig:growth_alpha} 
shows the effect on the exponent of varying $\Delta t$. As seen in that figure, larger 
$\Delta t$ increased the differences in exponent among different data sets, making the 
differences more distinguishable. Thus, $\Delta t$ had better  be as large as possible while 
keeping $\mu$ and $\sigma$ computable. For this article, we chose $\Delta t = 5620$, 
which was one of the $\Delta t$ values used in \figref{fig:growth_alpha}.}, i.e., each 
segment had 5620 words and the horizontal axis indicates the averaged frequency of a 
specific word within a segment of 5620 words.

The points at the upper right represent the most frequent words, whereas those at the 
lower left represent the least frequent. Although the plots exhibited different 
distributions, they could globally be considered roughly aligned in a power-law manner. 
This finding is non-trivial, as seen in other analyses based on Taylor's law 
\citep{taylor}. The exponent $\alpha$ was almost the same even though English and 
Chinese are different languages using different kinds of script.

As explained in \secref{subsec:method}, the Taylor exponent indicates the degree of 
consistent co-occurrence among words. The value of 0.58 obtained here suggests that 
the words of natural language texts are not strongly or consistently coherent with 
respect to each other. Nevertheless, the value is well above 0.5, and for the real data 
listed in \tabref{tab:data} (first to third blocks), not a single sample gave an exponent 
as low as 0.5.

Although the overall global tendencies in \figref{fig:fig1} followed power laws, many 
points deviated significantly from the regression lines. The words with the greatest 
fluctuation were often keywords. For example, among words in {\em Moby Dick} with 
large $\mu$, those with the largest $\sigma$ included {\em whale}, {\em captain}, 
and {\em sailor}, whereas those with the smallest $\sigma$ included functional words 
such as {\em to}, {\em that}, and {\em with}.

The Taylor exponent depended only slightly on the data size. \figref{fig:fig4} shows 
this dependency for the two largest data sets used, The New York Times (NYT, 1.5 
billion words) and The Mainichi (24 years) newspapers. When the data size was increased,
the exponent exhibited a slight tendency to decrease. For the NYT, the decrease seemed 
to have a lower limit, as the figure shows that the exponent stabilized at around 
$10^7$ words.

The reason for this decrease can be explained as follows. The Taylor exponent becomes 
larger when some words occur in a clustered manner. Making the text size larger 
increases the number of segments (since $\Delta t$ was fixed in this experiment). If the 
number of clusters does not increase as fast as the increase in the number of segments, 
then the number of clusters per segment becomes smaller, leading to a smaller exponent. 
In other words, the influence of each consecutive co-occurrence of a particular word 
decays slightly as the overall text size grows.

\begin{figure} [t]
	\centering
	\includegraphics[width=\columnwidth]{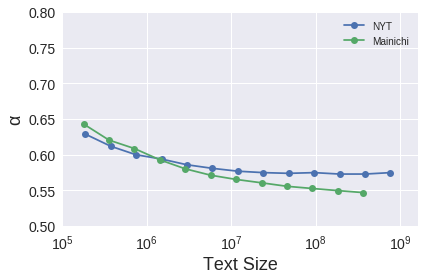}
	\vspace*{-0.8cm}
\caption{\label{fig:fig4} Taylor exponent $\hat{\alpha}$ (vertical axis) calculated for 
the two largest texts: The New York Times and The Mainichi newspapers. To evaluate the 
exponent's dependence on the text size, parts of each text were taken and the exponents 
were calculated for those parts, with points taken logarithmically. The window size was 
$\Delta t = 5620$. As the text size grew, the Taylor exponent slightly decreased.}
\end{figure}

\begin{figure}[t]
	\centering
	\subcaptionbox{enwiki8 (Wikipedia, tagged)\label{subfig:wiki}}{\includegraphics[clip, width=\pr\columnwidth]{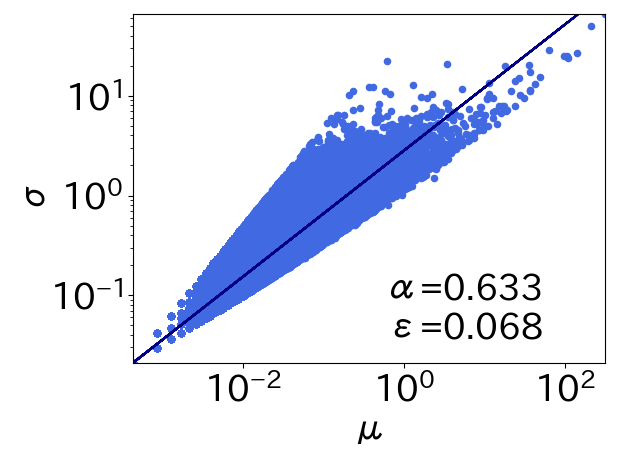}}%
	\subcaptionbox{Thomas (CHILDES)\label{subfig:thomas}}{\includegraphics[clip, width=\pr\columnwidth]{koromu_word_line2_w=5620.png}}
	\subcaptionbox{Lisp\label{subfig:lisp}}{\includegraphics[clip, width=\pr\columnwidth]{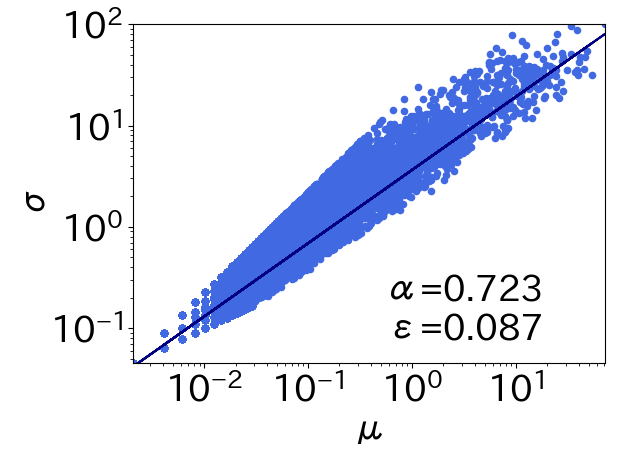}}%
	\subcaptionbox{Bach's {\em St Matthew Passion}\label{subfig:bach}}{\includegraphics[clip, width=\pr\columnwidth]{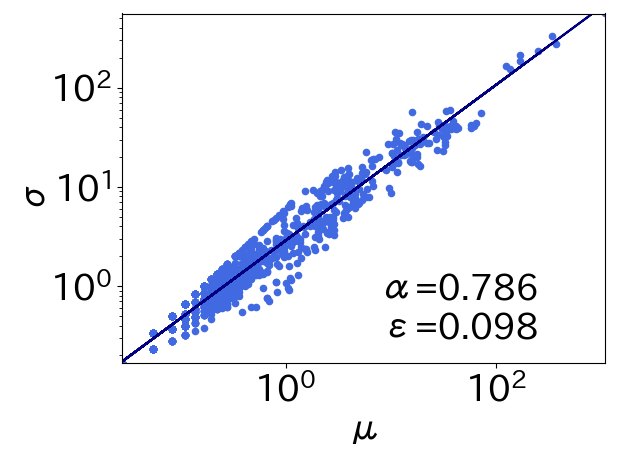}}
	\caption{\label{fig:fig2}
Examples of Taylor's law for alternative data sets listed in \tabref{tab:data}: 
enwiki8 (tag-annotated Wikipedia), Thomas (longest in CHILDES), Lisp source code, 
and the music of Bach's {\em St Matthew Passion}. These examples exhibited larger 
Taylor exponents than did typical natural language texts.}
\end{figure}

Analysis of different kinds of data showed how the Taylor exponent differed according 
to the data source. \figref{fig:fig2} shows plots for samples from enwiki8 (tagged 
Wikipedia), the child-directed speech of Thomas (taken from CHILDES), programming 
language data sets, and music. The distributions appear different from those for the 
natural language texts, and the exponents were significantly larger. This means that 
these data sets contained expressions with fixed forms much more frequently than did 
the natural language texts.

\begin{figure*} [t]
\centering
\includegraphics[width=0.8\textwidth]{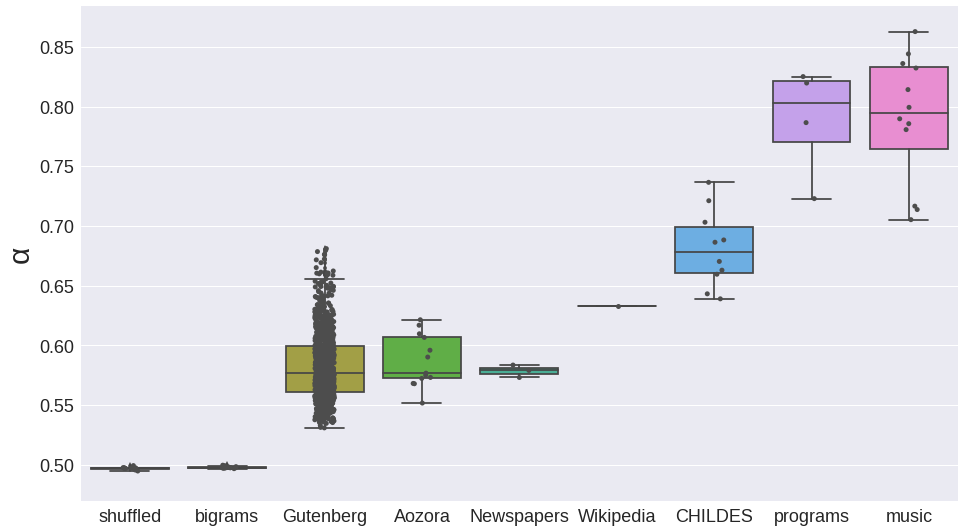}
\caption{Box plots of the Taylor exponents for different kinds of data. Each point 
represents one sample, and samples from the same kind of data are contained in each 
box plot. The first two boxes are for the randomized data, while the remaining boxes are 
for real data, including both the natural language texts and language-related sequences. 
Each box ranges between the quantiles, with the middle line indicating the median, the 
whiskers showing the maximum and minimum, and some extreme values lying 
beyond.}
\label{fig:box}
\end{figure*} 

\figref{fig:box} summarizes the overall picture among the different data sources. The 
median and quantiles of the Taylor exponent were calculated for the different kinds of 
data listed in \tabref{tab:data}. The first two boxes show results with an exponent of 
0.50. These results were each obtained from 10 random samples of the randomized 
sequences. We will return to these results in the next section.

The remaining boxes show results for real data. The exponents for texts from Project 
Gutenberg ranged from 0.53 to 0.68. \figref{fig:hist} shows a histogram of these texts 
with respect to the value of $\hat{\alpha}$. The number of texts decreased 
significantly at a value of 0.63, showing that the distribution of the Taylor exponent was 
rather tight. The kinds of texts at the upper limit of exponents for Project Gutenberg 
included structured texts of fixed style, such as dictionaries, lists of histories, and 
Bibles.

\begin{figure} [t]
\centering
\includegraphics[width=0.95\columnwidth]{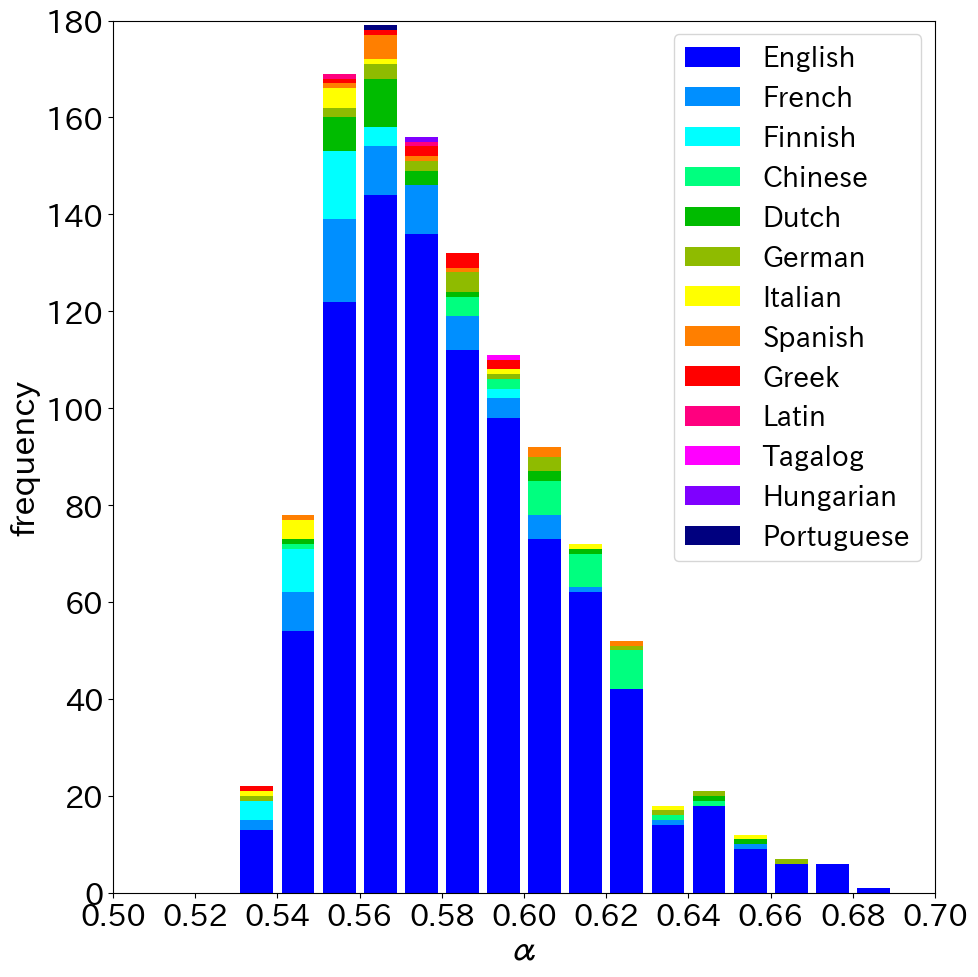}
\caption{\label{fig:hist} Histogram of Taylor exponents for long texts in Project 
Gutenberg (1129 texts). The legend indicates the languages, in frequency order. Each 
bar shows the number of texts with that value of $\hat{\alpha}$. Because of the skew 
of languages in the original conception of Project Gutenberg, the majority of the texts 
are in English, shown in blue, whereas texts in other languages are shown in other 
colors. The histogram shows how the Taylor exponent ranged fairly tightly around the 
mean, and natural language texts with an exponent larger than 0.63 were rare.}
\end{figure} 

The majority of texts were in English, followed by French and then other languages, as 
listed in \tabref{tab:data}. Whether $\alpha$ distinguishes languages is a difficult 
question. The histogram suggests that Chinese texts exhibited larger values than did 
texts in Indo-European languages. We conducted a statistical test to evaluate whether 
this difference was significant as compared to English. Since the numbers of texts were 
very different, we used the non-parametric statistical test of the Brunner-Munzel method, 
among various possible methods, to test a null hypothesis of whether $\alpha$ was 
equal for the two distributions \citep{bm}. The p-value for Chinese was 
$p=1.24\times10^{-16}$, thus rejecting the null hypothesis at the significance level of 
0.01. This confirms that $\alpha$ was generally larger for Chinese texts than for 
English texts. Similarly, the null hypothesis was rejected for Finnish and French, but it 
was accepted for German and Japanese at the 0.01 significance level. Since Japanese 
was accepted despite its large difference from English, we could not conclude whether 
the Taylor exponent distinguishes languages.

Turning to the last four columns of \figref{fig:box}, representing the enwiki8, 
child-directed speech (CHILDES), programming language, and music data, the Taylor 
exponents clearly differed from those of the natural language texts. Given the 
template-like nature of these four data sources, the results were somewhat expected. The 
kind of data thus might be distinguishable using the Taylor exponent. To confirm this, 
however, would require assembling a larger data set. Applying this approach with 
Twitter data and adult utterances would produce interesting results and remains for our 
future work.

\begin{figure} [t]
\centering
\includegraphics[width=\columnwidth]{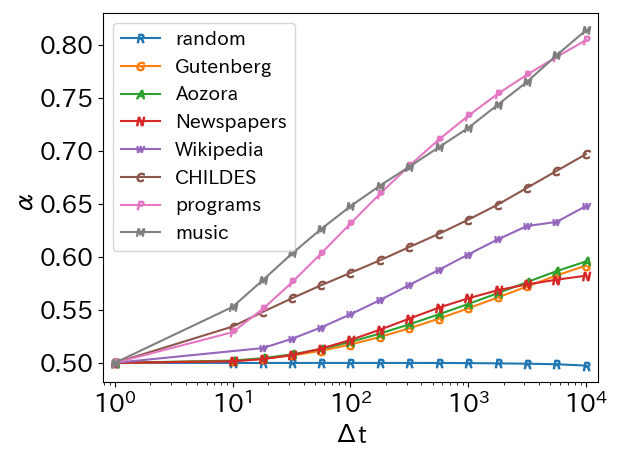}
\caption{Growth of $\hat{\alpha}$ with respect to $\Delta t$, averaged across data 
sets within each data kind. The plot labeled ``random'' shows the average for the two 
datasets of randomized text from Moby Dick (shuffled and bigrams, as explained in 
\secref{sec:lm}). Since this analysis required a large amount of computation, for the 
large data sets (such as newspaper and programming language data), 4 million words 
were taken from each kind of data and used here. When $\Delta t$ was small, the 
Taylor exponent was close to 0.5, as theoretically described in the main text. As $\Delta 
t$ was increased, the value of $\hat{\alpha}$ grew. The maximum $\Delta t$ was 
about 10,000, or about one-tenth of the length of one long literary text. For the kinds of 
data investigated here, $\hat{\alpha}$ grew almost linearly. The results show that, at a 
given $\Delta t$, the Taylor exponent has some capability to distinguish different kinds 
of text data.\label{fig:growth_alpha}}
\end{figure} 

The Taylor exponent also differed according to $\Delta t$, and 
\figref{fig:growth_alpha} shows the dependence of $\hat{\alpha}$ on $\Delta t$. For 
each kind of data shown in \figref{fig:box}, the mean exponent is plotted for various 
$\Delta t$. As reported in \citep{taylor}, the exponent is known to grow when the 
segment size gets larger. The reason is that words occur in a bursty, clustered manner at 
all length scales: no matter how large the segment size becomes, a segment will include 
either many or few instances of a given word, leading to larger variance growth. This 
phenomenon suggests how word co-occurrences in natural language are self-similar. 
The Taylor exponent is initially 0.5 when the segment size is very small. This can be 
analytically explained as follows \cite{taylor}. Consider the case of $\Delta t$=1. Let 
$n$ be the frequency of a particular word in a segment. We have $\average{n} \ll 1.0$, 
because the possibility of a specific word appearing in a segment becomes very small. 
Because $\average{n}^2 \approx 0$, $\sigma^2 = \average{n^2} - \average{n}^2 
\approx \average{n^2}$. Because $n=1$ or $0$ (with $\Delta t$=1), $\average{n^2} 
= \average{n} = \mu$. Thus, $\sigma^{2} \approx \mu$. 

Overall, the results show the possibility of applying Taylor's exponent to quantify the 
complexity underlying coherence among words. Grammatical complexity was 
formalized by Chomsky via the Chomsky hierarchy \citep{chomsky}, which describes 
grammar via rewriting rules. The constraints placed on the rules distinguish four 
different levels of grammar: regular, context-free, context-sensitive, and phrase structure. 
As indicated in \citep{badii}, however, this does not quantify the complexity on a {\em 
continuous} scale. For example, we might want to quantify the complexity of 
child-directed speech as compared to that of adults, and this could be addressed in only 
a limited way through the Chomsky hierarchy. Another point is that the hierarchy is 
sentence-based and does not consider fluctuation in the kinds of words appearing. 

\section{Evaluation of Machine-Generated Text by the Taylor 
Exponent}
\label{sec:lm}

\begin{figure}[t]
	\centering
	\subcaptionbox{Moby Dick (shuffled)\label{subfig:taylor_shuffled}}{\includegraphics[clip, width=\pr\columnwidth]{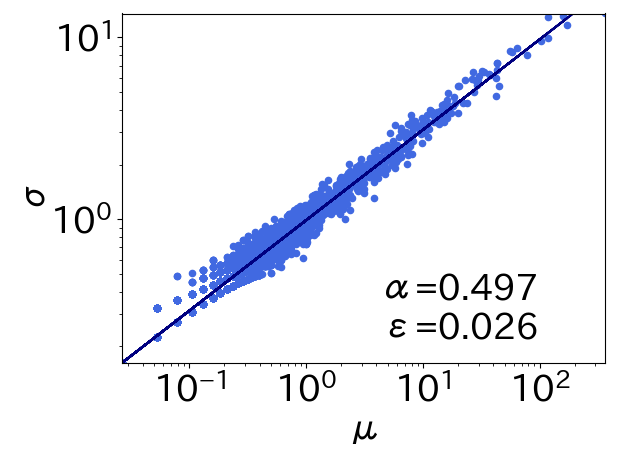}}%
	\subcaptionbox{Moby Dick (bigram)\label{subfig:taylor_bigram}}{\includegraphics[clip, width=\pr\columnwidth]{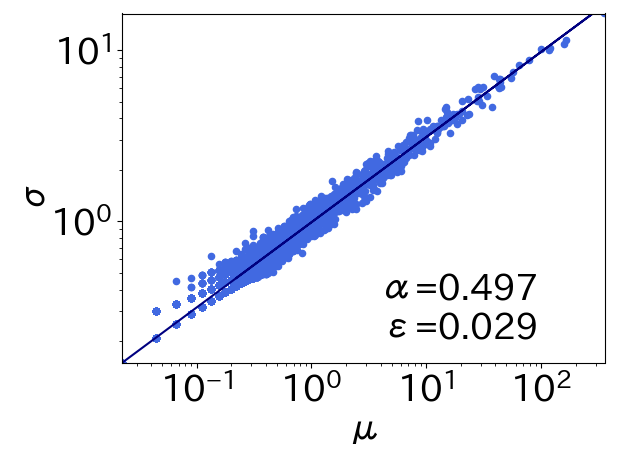}}
	\caption{\label{fig:random}
Taylor analysis of a shuffled text of {\em Moby Dick} and a randomized text 
generated by a bigram model. Both exhibited an exponent of 0.50.}
\end{figure}

The main contribution of this paper is the findings of Taylor's law behavior for real texts 
as presented thus far. This section explains the applicability of these findings, through 
results obtained with baseline language models.

As mentioned previously, i.i.d. mathematical processes have a Taylor exponent of 0.50. 
We show here that, even if a process is not trivially i.i.d., the exponent often takes a 
value of 0.50 for random processes, including texts produced by standard language 
models such as $n$-gram based models. A more complete work in this direction is 
reported in \citep{takahashi}.

\figref{fig:random} shows samples from each of two simple random processes. 
\figref{subfig:taylor_shuffled} shows the behavior of a shuffled text of {\em Moby 
  Dick}. Obviously, since the sequence was almost i.i.d. following Zipf distribution, 
the Taylor 
exponent was 0.50. Given that the Taylor exponent becomes larger for a sequence with 
words dependent on each other, as explained in \secref{sec:theory}, we would expect 
that a sequence generated by an $n$-gram model would exhibit an exponent larger than 
0.50. The simplest such model is the bigram model, so a sequence of 300,000 words 
was probabilistically generated using a bigram model of {\em Moby Dick}. 
\figref{subfig:taylor_bigram} shows the Taylor analysis, revealing that the exponent 
remained 0.50.

This result does not depend much on the quality of the individual samples. The first and 
second box plots in \figref{fig:box} show the distribution of exponents for 10 different 
samples for the shuffled and bigram-generated texts, respectively. The exponents were 
all around 0.50, with small variance.

\begin{figure}[t]
	\centering
	\subcaptionbox{Text produced by LSTM (3-layer stacked character-based)\label{subfig:lstm}}{\includegraphics[clip, width=\pr\columnwidth]{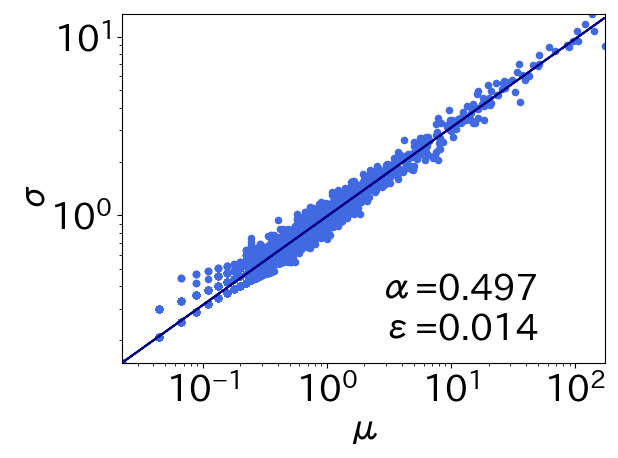}}%
	\subcaptionbox{Machine-translated text using neural language model\label{subfig:mt}}{\includegraphics[clip, width=\pr\columnwidth]{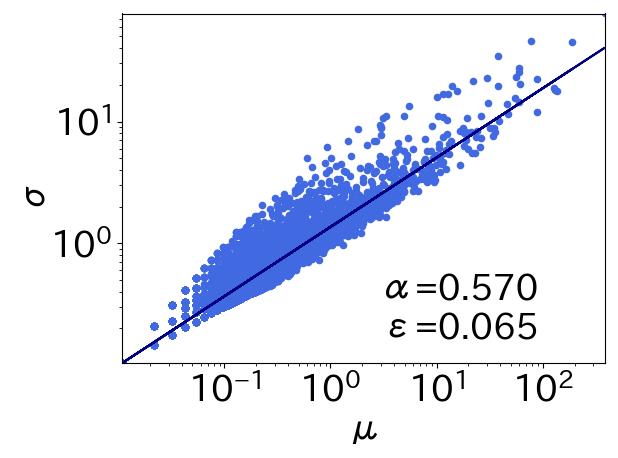}}
	\caption{\label{fig:dl}
Taylor analysis for two texts produced by standard neural language models: 
\subref{subfig:lstm} a stacked LSTM model that learned the complete works of 
Shakespeare; and \subref{subfig:mt} a machine translation of {\em Les 
Mis\'{e}rables} (originally in French, translated into English), from a neural language 
model. }
\end{figure}

State-of-the-art language models are based on neural models, and they are mainly 
evaluated by perplexity and in terms of the performance of individual applications. 
Since their architecture is complex, quality evaluation has become an issue. One 
possible improvement would be to use an evaluation method that qualitatively differs 
from judging application performance. One such method is to verify whether the 
properties underlying natural language hold for texts generated by language models. 
The Taylor exponent is one such possibility, among various properties of natural 
language texts.

As a step toward this approach, \figref{fig:dl} shows two results produced by neural 
language models. \figref{subfig:lstm} shows the result for a sample of 2 million 
characters produced by a standard (three-layer) stacked character-based LSTM unit that 
learned the complete works of Shakespeare. The model was optimized to minimize the 
cross-entropy with a stochastic gradient algorithm to predict the next character from the 
previous 128 characters. See \cite{plos18} for the details of the experimental settings. 
The Taylor exponent of the generated text was 0.50. This indicates that the 
character-level language model could not capture or reproduce the word-level clustering 
behavior in text. This analysis sheds light on the quality of the language model, separate 
from the prediction accuracy. 

The application of Taylor's law for a wider range of language models appears in 
\citep{takahashi}. Briefly, state-of-the-art word-level language models can generate 
text whose Taylor exponent is larger than 0.50 but smaller than that of the dataset used 
for training. This indicates both the capability of modeling burstiness in text and the 
room for improvement. Also, the perplexity values correlate well with the Taylor 
exponents. Therefore, Taylor exponent can reasonably serve for evaluating 
machine-generated text.

In contrast to character-level neural language models, neural-network-based machine 
translation (NMT) models are, in fact, capable of maintaining the burstiness of the 
original text. \figref{subfig:mt} shows the Taylor analysis for a machine-translated text 
of {\em Les Mis\'{e}rables} (from French to English), obtained from Google NMT 
\cite{Wu_2016}. We split the text into 5000-character portions because of the API's 
limitation (See \citep{plos18} for the details). As is expected and desirable, the translated text retains the clustering 
behavior of the original text, as the Taylor exponent of 0.57 is equivalent to that of the 
original text.

\section{Conclusion}
We have proposed a method to analyze whether a natural language text follows Taylor's 
law, a scaling property quantifying the degree of consistent co-occurrence among words. 
In our method, a sequence of words is divided into given segments, and the mean and 
standard deviation of the frequency of every kind of word are measured. The law is 
considered to hold when the standard deviation varies with the mean according to a 
power law, thus giving the Taylor exponent.

Theoretically, an i.i.d. process has a Taylor exponent of 0.5, whereas larger exponents 
indicate sequences in which words co-occur systematically. Using over 1100 texts 
across 14 languages, we showed that written natural language texts follow Taylor's law, 
with the exponent distributed around 0.58. This value differed greatly from the 
exponents for other data sources: enwiki8 (tagged Wikipedia, 0.63), child-directed 
speech (CHILDES, around 0.68), and programming language and music data (around 
0.79). These Taylor exponents imply that a written text is more complex than 
programming source code or music with regard to fluctuation of its components. None 
of the real data exhibited an exponent equal to 0.5. We conducted more detailed analysis 
varying the data size and the segment size.

Taylor's law and its exponent can also be applied to evaluate machine-generated text. 
We showed that a character-based LSTM language model generated text with a Taylor 
exponent of 0.5. This indicates one limitation of that model.

Our future work will include an analysis using other kinds of data, such as Twitter data 
and adult utterances, and a study of how Taylor's law relates to grammatical complexity 
for different sequences. Another direction will be to apply fluctuation analysis in 
formulating a statistical test to evaluate the structural complexity underlying a sequence. 

\section*{Acknowledgments}
This work was supported by JST Presto Grant Number JPMJPR14E5 and HITE funding. 
We thank Shuntaro Takahashi for offering his comments and providing the 
machine-generated data reported in \secref{sec:lm}.

\appendix
\section{Mathematical Proof of Taylor Exponent}
\label{proof}
Here we show that the Taylor exponent of an independent and identically distributed 
(i.i.d.) process is 0.5. A proof in a more general form is shown in \citep{taylor}. This is 
a known mathematical fact, as found previously in \citep{yule}.

\begin{prop}
The Taylor exponent of a sequence generated by an i.i.d. process is 0.5.
\end{prop}
\begin{proof}
Consider i.i.d. random variables $X_1, \ldots, X_i, \ldots, X_N$, where $i$ denotes the 
location within a text. For a specific word $w_k \in W$, with $W$ being the set of 
words, let $p_k$ denote the probability of occurrence of word $w_k$, i.e., 
$\mathbb{P}(X_i = w_k) = p_k$ (for all $i$). Naturally, the expectation 
$\mathbb{E}$ and variance $\mathbb{V}$ of the count of $w_{k}$ for 
$X_{i}$ are the following:
\begin{eqnarray}
	\mathbb{E}[X_i] &=& p_k, \label{eq:expectation} \\
	\mathbb{V}[X_i] &=& p_k(1 - p_k),
\end{eqnarray}
which only depend on the constant $p_k$. With window size $\Delta t$, $\mu_k = 
\Delta t\mathbb{E}[X_i]$. Note that $\sigma_k^2 = \Delta t\mathbb{V}[X_i]$, 
because
\begin{eqnarray*}
	\sigma_{k}^{2} &=& \mathbb{V} \left[\sum_{i=1}^{\Delta t} X_i\right] \\
	&=& \mathbb{E} \left[\left(\sum_{i=1}^{\Delta t} (X_i - p_k)\right)^2\right] \\
	&=& \mathbb{E} \left[\sum_{i=1}^{\Delta t} (X_i - p_k)^2\right.\\
	&& \quad\left. + 2\sum_{i \neq j}(X_i - p_k)(X_j - p_k)\right] \\
 &=&\mathbb{E} \left[\sum_{i=1}^{\Delta t} (X_i - p_k)^2\right]\\
	&=& \sum_{i=1}^{\Delta t}\mathbb{V}[X_i]\\
	&=& \Delta t \mathbb{V}[X_{i}].
\end{eqnarray*}
Furthermore, note that $\mathbb{E}[(X_i - p_k)(X_j - p_k)] = 0$ for every $i, j$ with 
$i\neq j$, because $X_{i}$ and $X_{j}$ are independent of each other and 
\eqref{eq:expectation} holds. Therefore, Taylor exponent $\alpha$ of an i.i.d. process 
is 0.5, because
\begin{equation*}
	\sigma_k^2 = \frac{\mathbb{V}[X_i]}{\mathbb{E}[X_i]} \mu_k.
\end{equation*}
\end{proof}

\bibliographystyle{compling}
\bibliography{taylor}

\end{document}